\documentclass[sigconf]{aamas}  
\usepackage{balance}  

\usepackage{booktabs}
\usepackage{multirow}
\usepackage{subcaption}
\usepackage[noend]{algpseudocode}
\usepackage{algorithm}

\usepackage{tikz}
\usetikzlibrary{fit}
\usetikzlibrary{positioning, arrows, matrix}

\setcopyright{none}  
\acmDOI{}  
\acmISBN{}  
\acmConference[AAMAS'19]{Proc.\@ of the 18th International Conference on Autonomous Agents and Multiagent Systems (AAMAS 2019)}{May 13--17, 2019}{Montreal, Canada}{N.~Agmon, M.~E.~Taylor, E.~Elkind, M.~Veloso (eds.)}  
\acmYear{2019}  
\copyrightyear{2019}  
\acmPrice{}  

\begin{document}

\title{Information Gathering in Decentralized POMDPs by Policy Graph Improvement}  


\author{Mikko Lauri}
\affiliation{%
  \institution{University of Hamburg}
  \city{Hamburg} 
  \state{Germany} 
  \postcode{20255}
}
\email{lauri@informatik.uni-hamburg.de}
\author{Joni Pajarinen}
\affiliation{%
  \institution{TU Darmstadt}
  \city{Darmstadt} 
  \state{Germany} 
  \postcode{64289}
}
\email{pajarinen@ias.tu-darmstadt.de}

\author{Jan Peters}
\affiliation{%
  \institution{TU Darmstadt}
  \city{Darmstadt} 
  \state{Germany} 
  \postcode{64289}
}
\email{peters@ias.tu-darmstadt.de}

\begin{abstract}  
Decentralized policies for information gathering are required when multiple autonomous agents are deployed to collect data about a phenomenon of interest without the ability to communicate.
Decentralized partially observable Markov decision processes (Dec-POMDPs) are a general, principled model well-suited for such decentralized multiagent decision-making problems.
In this paper, we investigate Dec-POMDPs for decentralized information gathering problems.
An optimal solution of a Dec-POMDP maximizes the expected sum of rewards over time.
To encourage information gathering, we set the reward as a function of the agents' state information, for example the negative Shannon entropy.
We prove that if the reward is convex, then the finite-horizon value function of the corresponding Dec-POMDP is also convex.
We propose the first heuristic algorithm for information gathering Dec-POMDPs, and empirically prove its effectiveness by solving problems an order of magnitude larger than previous state-of-the-art.
\end{abstract}



\maketitle


\section{Introduction}
Autonomous agents and robots can be deployed in information gathering tasks in environments where human presence is either undesirable or infeasible.
Examples include monitoring of deep ocean conditions, or space exploration.
It may be desirable to deploy a team of agents, e.g., due to the large scope of the task at hand, resulting in a decentralized information gathering task.

Some recent works, e.g.,~\cite{Charrow2014,Schlotfeldt2018}, tackle decentralized information gathering while assuming perfect, instantaneous communication between agents, while centrally planning how the agents should act.
In terms of communication, we approach the problem from the other extreme as a decentralized partially observable Markov decision process (Dec-POMDP)~\cite{Oliehoek2016}.
In a Dec-POMDP, no explicit communication between the agents is assumed\footnote{If desired, communication may be included into the Dec-POMDP model~\cite{Spaan2006,Wu2011}.}. 
Each agent acts independently, without knowing what the other agents have perceived or how they have acted.

Informally, a Dec-POMDP model consists of a set of agents in an environment with a hidden state. 
Each agent has its own set of local actions, and a set of local observations it may observe.
Markovian state transition and observation processes conditioned on the agents' actions and the state determine the relative likelihoods of subsequent states and observations.
A reward function determines the utility of executing any action in any state.
The objective is to centrally design optimal control policies for each agent that maximize the expected sum of rewards over a finite horizon of time.
The control policy of each agent depends only on the past actions and observations of that agent, hence no communication during execution of the policies is required.
However, as policies are planned centrally, it is possible to reasong about the joint information state of all the agents.
It is thus possible to calculate probability distributions over the state, also known as joint beliefs.

A decentralized information gathering task differs from other multiagent control tasks by the lack of a goal state.
It is not the purpose of the agents to execute actions that reach a particular state, but rather to observe the environment in a manner that provides the greatest amount of information while satisfying operational constraints.
As the objective is information acquisition, the reward function depends on the joint belief of the agents.
Convex functions of a probability mass function naturally model certainty~\cite{DeGroot2004}, and have been proposed in the context of single-agent POMDPs~\cite{Araya2010} and Dec-POMDPs~\cite{Lauri2017}.
However, to the best of our knowledge no heuristic or approximate algorithms for convex reward Dec-POMDPs have been proposed, and no theoretical results on the properties of such Dec-POMDPs exist in the literature.

In this paper, we propose the first heuristic algorithm for Dec-POMDPs with a convex reward.
We prove the value function of such Dec-POMDPs is convex, generalizing the similar result for single-agent POMDPs~\cite{Araya2010}.
The Dec-POMDP generalizes other decision-making formalisms such as multi-agent POMDPs and Dec-MDPs~\cite{Bernstein2002}.
Thus, our results also apply to these special cases removing parts required by the more general Dec-POMDP.

Our paper has three contributions.
Firstly, we prove that in Dec-POMDPs where the reward is a convex function of the joint belief, the value function of any finite horizon policy is convex in the joint belief.
Secondly, we propose the first heuristic algorithm for Dec-POMDPs with a reward that is a function of the agents' joint state information.
The algorithm is based on iterative improvement of the value of fixed-size policy graphs.
We derive a lower bound that may be improved instead of the exact value, leading to computational speed-ups.
Thirdly, we experimentally verify the feasibility and usefulness of our algorithm.
For Dec-POMDPs with a state information dependent reward, we find policies for problems an order of magnitude larger than previously.

The paper is organized as follows.
We review related work in Section~\ref{sec:related_work}.
In Section~\ref{sec:decentralized_information_gathering}, we define our Dec-POMDP problem and introduce notation and definitions.
Section~\ref{sec:value_of_a_policy_node} derives the value of a policy graph node.
In Section~\ref{sec:properties_of_information_gathering_Dec-POMDPs}, we prove convexity of the value in a Dec-POMDP where the reward is a convex function of the state information.
Section~\ref{sec:policy_graph_improvement} introduces our heuristic policy improvement algorithm.
Experimental results are presented in Section~\ref{sec:experiments}, and concluding remarks are provided in Section~\ref{sec:conclusion}.

\section{Related work} 
\label{sec:related_work}
Computationally finding an optimal decentralized policy for a finite-horizon Dec-POMDP is NEXP-complete~\cite{Bernstein2002}.
Exact algorithms for Dec-POMDPs are usually based either on backwards in time dynamic programming~\cite{Hansen2004}, forwards in time heuristic search~\cite{Szer2005,Oliehoek2013}, or on exploiting the inherent connection of Dec-POMDPs to non-observable Markov decision processes~\cite{Dibangoye2016,MacDermed2013}.
Approximate and heuristic methods have been proposed, e.g., based on finding locally
optimal ``best response'' policies for each agent~\cite{Nair2003}, memory-bounded dynamic programming~\cite{Seuken2007}, cross-entropy optimization over the space of policies~\cite{Oliehoek2008}, or monotone iterative improvement of fixed-size policies~\cite{Pajarinen2011}.
Algorithms for special cases such as goal-achievement Dec-POMDPs~\cite{Amato2009} and factored Dec-POMDPs, e.g.,~\cite{Oliehoek2008b}, have also been proposed.
Structural properties, such as transition, observation, and reward independence between the agents, can also be leveraged and may even result in a problem with a lesser computational complexity~\cite{allen2009complexity}.
Some Dec-POMDP algorithms~\cite{Oliehoek2013} take advantage of plan-time sufficient statistics, which are joint distributions over the hidden state and the histories of the agents' actions and observations~\cite{Oliehoek2013b}.
The sufficient statistics provide a means to reason about possible distributions over the hidden state, also called joint beliefs, reached under a given policy.

The expected value of a reward function that depends on the hidden state and action is a linear function of the joint belief.
These types of rewards are standard in Dec-POMDPs.
In the context of single-agent POMDPs, Araya-L\'{o}pez et al.~\cite{Araya2010} argue that information gathering tasks are naturally formulated using a reward function that is a convex function of the state information and introduce the $\rho$POMDP model with such a reward.
This enables application of, e.g., the negative Shannon entropy of the state information as a component of the reward function.
Under certain conditions, an optimal value function of a $\rho$POMDP is Lipschitz-continuous~\cite{Fehr2018} which may be exploited in a solution algorithm.
An alternative formulation for information gathering in single-agent POMDPs is presented in~\cite{Spaan2015}, and its connection to $\rho$POMDPs is characterized in~\cite{Satsangi2018}.
Recently, \cite{Lauri2017} proposes an extension of the ideas presented in~\cite{Araya2010} to the Dec-POMDP setting.
Entropy is applied in the reward function to encourage information gathering.
Problem domains with up to 25 states and 5 actions per agent are solved with an exact algorithm.

In this paper, we present the first heuristic algorithm for Dec-POMDPs with rewards that depend non-linearly on the joint belief.
Our algorithm is based on the combination of the idea of using a fixed-size policy represented as a graph~\cite{Pajarinen2011} with plan-time sufficient statistics~\cite{Oliehoek2013b} to determine joint beliefs at the policy graph nodes.
The local policy at each policy graph node is then iteratively improved, monotonically improving the value of the node.
We show that if the reward function is convex in the joint belief, then the value function of any finite-horizon Dec-POMDP policy is convex as well.
This is a generalization of a similar result known for single-agent POMDPs~\cite{Araya2010}.
From this property, we obtain a lower bound for the value of a policy that we empirically show improves the efficiency of our algorithm.
Compared to prior state-of-the-art in Dec-POMDPs with convex rewards~\cite{Lauri2017}, our algorithm is capable of handling problems an order of magnitude larger.

\section{Decentralized POMDPs} 
\label{sec:decentralized_information_gathering}
We next formally define the Dec-POMDP problem we consider.
Contrary to most earlier works, we define the reward as a function of \emph{state information} and action.
This allows us to model information acquisition problems.
We choose the finite-horizon formulation to reflect the fact that a decentralized information gathering task should have a clearly defined end after which the collected information is pooled and subsequent inference or decisions are made.

A finite-horizon Dec-POMDP is a tuple $\big(I$, $S$, $\{A_i\}$, $\{Z_i\}$, $P^s$, $P^z$, $b^0$, $T$, $\{\rho_t\}\big)$, where $I=\{1, \ldots, n\}$ is the set of agents, $S$ is a finite set of hidden states, $A_i$ and $Z_i$ are the finite action and observation sets of agent $i \in I$, respectively, $P^s$ is the state transition probability that gives the conditional probability $P^s(s^{t+1}\mid s^t, a^t)$ of the new state $s^{t+1}$ given the current state $s^t$ and joint action $a^t = (a_1^t, \ldots, a_n^t) \in A$, where $A$ is the joint action space obtained as the Cartesian product of $A_i$ for all $i\in I$, $P^z$ is the observation probability that gives the conditional probability $P^z(z^{t+1}\mid s^{t+1}, a^t)$ of the joint observation $z^{t+1}=(z_1^{t+1}, \ldots z_n^{t+1})\in Z$ given the state $s^{t+1}$ and previous joint action $a^t$, with $Z$ being the joint observation space defined as the Cartesian product of $Z_i$ for $i\in I$, $b^0 \in \Delta(S)$ is the initial state distribution\footnote{We denote by $\Delta(S)$ the space of probability mass functions over $S$.} at time $t=0$, $T \in \mathbb{N}$ is the problem horizon, and $\rho_t: \Delta(S)\times A \to \mathbb{R}$ are the reward functions at times $t=0, \ldots, T-1$, while $\rho_T:\Delta(S) \to \mathbb{R}$ determines a final reward obtained at the end of the problem horizon.

The Dec-POMDP starts from some state $s^0 \sim b^0$.
Each agent $i\in I$ then selects an action $a_i^0 \in A_i$, and the joint action $a^0=(a_1^0, \ldots, a_n^0)\in A$ is executed.
The state then transitions according to $P^s$, and each agent perceives an observation $z_i^1 \in Z_i$, where the likelihood of the joint observation $z^1=(z_1^1, \ldots, z_n^1) \in Z$ is determined according to $P^z$.
The agents then select the next actions $a_i^1$, and the same steps are repeated until $t=T$ and the task ends.

Optimally solving a Dec-POMDP means to design a policy for each agent that encodes which action the agent should execute conditional on its past observations and actions; in a manner such that the expected sum of rewards collected is maximized.
In the following, we make the notion of a policy exact, and determine the expected sum of rewards collected when executing a policy.

\subsection{Histories and policies} 
\label{sub:histories_and_policies}


Define the history set of agent $i$ at time $t = 1, \ldots, T$ as $H_i^t = \{ (b^0, a_i^0, z_i^1, \ldots, a_i^{t-1}, z_i^{t}) \mid a_i^k \in A_i, z_i^k \in Z_i\}$, and $H_i^0 = \{(b^0)\}$.
A local history $h_i^t \in H_i^t$ contains all information available to agent $i$ to decide its next action $a_i^t$.
We define the joint history set $H^t$ as the Cartesian product of $H_i^t$ over $i\in I$.
We write a joint history as $h^t = (h_1^t, \ldots, h_n^t) \in H^t$, or equivalently as $h^t=(b_0, a^0, z^1, \ldots, a^{t-1},z^{t}) \in H^t$ where $a^k\in A$ and $z^k \in Z$.
Both the local and joint histories satisfy the recursion $h^t = (h^{t-1}, a^{t-1}, z^t)$.

\begin{figure}[t]
\begin{center}
  \begin{tikzpicture}[main node/.style={circle,fill=gray!10,draw,minimum size=0.5cm,inner sep=2pt},
            edge/.style={->,> = latex'}, node distance = 1.5cm, auto]
      \node[main node] (0) {$q_i^{0}$};
      \node[main node] (1) [above right of=0] {$q_i^1$};
      \node[main node] (2) [below right of=0] {$q_i^2$};
      \node[main node] (3) [right of=1] {$q_i^3$};
      \node[main node] (4) [right of=2] {$q_i^4$};

      \draw[edge] (0) -- (1) node[midway] {$z_i^0$};
      \draw[edge] (0) -- (2) node[midway] {$z_i^1$};
      \draw[edge] (1) -- (3) node[midway] {$z_i^0$};
      \draw[edge] (2) -- (3) node[midway,left] {$z_i^1$};
      \draw[edge] (1) -- (4) node[midway,right] {$z_i^0$};
      \draw[edge] (2) -- (4) node[midway] {$z_i^1$};

      \node (tab1) [below right of=3, xshift=0.5cm] {%
        \begin{tabular}{cc}
        \toprule
        $q_i$ & $\gamma_i(q_i)$\\
        \midrule
        $q_i^0$ & $a_i^0$\\
        $q_i^1$ & $a_i^1$\\
        $q_i^2$ & $a_i^1$\\
        $q_i^3$ & $a_i^0$\\
        $q_i^4$ & $a_i^0$\\
        \bottomrule
        \end{tabular}};
   \end{tikzpicture} 
\end{center}
\caption{A local policy for agent $i$. The policy encodes the agent's behavior conditional on local observations. The shaded circles show the set of nodes $Q_i$. The starting node is $q_0^i$. The table on the right defines the output function $\gamma_i$, and the labels on the edges define the node transition function $\lambda_i$. First, the agent executes $\gamma_i(q_i^0)$. Conditional on the next observation perceived, the next node is $q_i^1$ or $q_i^2$. At the next node, the action to execute is again looked up from $\gamma_i$.}
\label{fig:pg_example}
\end{figure}
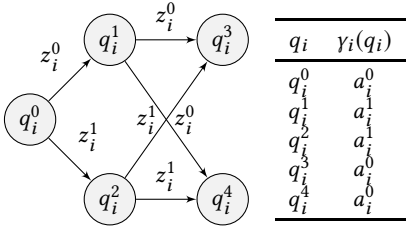

A solution of a finite-horizon Dec-POMDP is a local policy for each agent that determines which action an agent should take given a local history in $H_i^t$ for any $t=0, \ldots, T-1$.
We define a local policy similarly as~\cite{Pajarinen2011} as a deterministic finite-horizon controller viewed as a directed acyclic graph.
\begin{definition}[Local policy]
  For agent $i$, a local policy is $\pi_i=(Q_i, q_i^0, \gamma_i, \lambda_i)$, where $Q_i$ is a finite set of nodes, $q_i^0 \in Q_i$ is a starting node, $\gamma_i: Q_i\to A_i$ is an output function, and $\lambda_i: Q_i\times Z_i \to Q_i$ is a node transition function.
\end{definition}
Fig.~\ref{fig:pg_example} shows an example of a local policy.
Note that a sufficiently large graph can represent any finite horizon local policy.

We constrain the structure of local policies by enforcing that each node can be identified with a unique time step.
We call this the property of temporal consistency.
\begin{definition}[Temporal consistency]
\label{asm:temporal}
A local policy $\pi_i=(Q_i$, $q_i^0$, $\gamma_i$, $\lambda_i)$ is temporally consistent if $Q_i = \bigcup_{t=0}^{T-1} Q_i^t$ where $Q_i^t$ are pairwise disjoint and non-empty, and $Q_i^0 = \{q_i^0\}$, and for any $t=0, \ldots, T-2$, for $q_i^t \in Q_i^t$, for all $z_i\in Z_i$, $\lambda_i(q_i^t, z_i) \in Q_i^{t+1}$.
\end{definition}
In a temporally consistent policy, at a node in $Q_i^t$ the agent has $(T-t)$ decisions left until the end of the problem horizon.
Temporal consistency guarantees that exactly one node in each set $Q_i^t$ can be visited, and that after visiting a node in $Q_i^t$, the next node will belong to $Q_i^{t+1}$.
In Fig.~\ref{fig:pg_example}, $T=3$, and $Q_i^0=\{q_i^0\}$, $Q_i^1=\{q_i^1,q_i^2\}$, $Q_i^2=\{q_i^3,q_i^4\}$.
Temporal consistency is assumed throughout the rest of the paper.

A joint policy describes the joint behaviour of all agents and is defined as the combination of the local policies $\pi_i$.
\begin{definition}[Joint policy]
  Given local policies $\pi_i=(Q_i$, $q_i^0$, $\gamma_i$, $\lambda_i)$ for all $i \in I$, a joint policy is $\pi=(Q, q^0, \gamma, \lambda)$, where $Q$ is the Cartesian product of all $Q_i$, $q^0 = (q_1^0, \ldots, q_n^0) \in Q$, and for $q=(q_1, \ldots, q_n)\in Q$ and $z=(z_1,\ldots, z_n)\in Z$, $\gamma:Q\to A$ is such that $\gamma(q) = (\gamma_1(q_1), \ldots, \gamma_n(q_n))$, and $\lambda:Q\times Z \to Q$ is such that $\lambda(q,z) = (\lambda_1(q_1,z_1), \ldots, \lambda_n(q_n,z_n))$.
\end{definition}
Temporal consistency naturally extends to joint policies, such that there exists a partition of $Q$ by pairwise disjoint sets $Q^t$.

\subsection{Bayes filter} 
\label{sub:bayes_filter}
While planning policies for information gathering, it is useful to reason about the joint belief of the agents given some joint history.
This can be done via Bayesian filtering as described in the following.

The initial state distribution $b^0$ is a function of the state at time $t=0$, and for any state $s^0 \in S$, $b^0(s^0)$ is equal to the probability $P(s^0 \mid h^0)$.
When action $a^0$ is executed and observation $z^{1}$ is perceived, we may find the posterior belief $P(s^1 \mid h^1)$ where $h^{1}=(h^0, a^0, z^1)$ by applying a Bayes filter.

In general, given any current joint belief $b^t$ corresponding to some joint history\footnote{For notational convenience, we drop the explicit dependence of $b^t$ on the joint history.} $h^t$, and a joint action $a^t$ and joint observation $z^{t+1}$, the posterior joint belief is calculated by
\begin{equation}
  \label{eq:bayes_filter}
  b^{t+1}(s^{t+1}) = \frac{P^z(z^{t+1}\mid s^{t+1},a^{t})\sum\limits_{s^{t} \in S}P^s(s^{t+1}\mid a^{t},s^{t})b^{t}(s^{t})}{\eta(z^{t+1} \mid b^{t}, a^{t})},
\end{equation}
where $\eta(z^{t+1} \mid b^{t}, a^{t})$ is the normalization factor equal to the prior probability of observing $z^{t+1}$.
Given $b^0$ and any joint history $h^t=(b^0$,$a^0$,$z^1$, $\ldots$, $a^{t-1}$, $z^t)$,  repeatedly applying Eq.~\eqref{eq:bayes_filter} yields a sequence $b^0, b^1, \ldots, b^t$ of joint beliefs.
We shall denote the application of the Bayes filter by the shorthand notation $b^{t+1} = \zeta(b^t, a^t, z^{t+1})$.
Furthermore, we shall denote the filter that recovers $b^t$ given $h^t$ by repeated application of $\zeta$ by a function $\tau: H^t \to \Delta(S)$.

\subsection{Value of a policy} 
\label{sub:value_functions}
The value of a policy $\pi$ is equal to the expected sum of rewards collected when acting according to the policy.
We define value functions $V_t^\pi:\Delta(S) \times Q^{t} \to \mathbb{R}$ that give the expected sum of rewards when following policy $\pi$ until the end of the horizon when $t$ decisions have been taken so far, for any joint belief $b\in\Delta(S)$ and any policy node $q\in Q^{t}$.

The time step $t=T$ is a special case when all actions have already been taken, and the value function only depends on the joint belief and is equal to the final reward: $V_T(b) = \rho_T(b)$.

For $t=T-1$, one decision remains, and the remaining expected sum of rewards of executing policy $\pi$ is equal to 
\begin{equation}
  \label{eq:first_value_fun}
  V_{T-1}^\pi(b,q) \!= \!\rho_{T-1}(b,\gamma(q)) + \!\!\!\sum\limits_{z\in Z}\eta(z\mid b, \gamma(q)) V_T\!\left(\zeta(b, \gamma(q), z)\right),
\end{equation}
i.e., the sum of the immediate reward and the expected final reward at time $T$.
From the above, we define $V_t^\pi$ iterating backwards in time for $t=T-2, \ldots, 0$ as
\begin{equation}
\label{eq:value}
  V_t^\pi(b,q) =  \rho_t(b, \gamma(q)) + \mathbb{E}\left[V_{t+1}^\pi\left(\zeta(b, \gamma(q), z), \lambda(q,z)\right) \right],
\end{equation}
where the expectation is under $z \sim \eta(z \mid b, \gamma(q))$.
The expected sum of rewards collected when following a policy $\pi$ is equal to its value $V_0^\pi(b^0, q^0)$.
The objective is to find an optimal policy $\pi^*$ whose value is greater than or equal to the value of any other policy.

\section{Value of a policy node} 
\label{sec:value_of_a_policy_node}
Executing a policy corresponds to a stochastic traversal of the policy graphs (Fig.~\ref{fig:pg_example}) conditional on the observations perceived.
In this section, we first answer two questions related to this traversal process.
First, given a history, when is it consistent with a policy, and which nodes in the policy graph will be traversed (Subsection~\ref{sub:history_consistency})?
Second, given an initial state distribution, what is the probability of reaching a given policy graph node, and what are the relative likelihoods of histories if we assume a given node is reached (Subsection~\ref{sub:node_reachability_probabilities})?
With the above questions answered, we define the value of a policy graph node both in a joint and in a local policy (Subsection~\ref{sub:value_of_policy_nodes}).
These values will be useful in designing a policy improvement algorithm for Dec-POMDPs.

\subsection{History consistency} 
\label{sub:history_consistency}
As illustrated in Fig.~\ref{fig:pg_example}, there can be multiple histories along which a node can be reached.
We define when a history is consistent with a policy, i.e., when executing a policy could have resulted in the given history.
As histories in $H^T$ are reached after executing all actions, in the remainder of this subsection we consider $0 \leq t \leq T-1$.

\begin{definition}[History consistency]
  We are given for all $i\in I$ $\pi_i=$$(Q_i$,$q_i^0$,$\gamma_i$,$\lambda_i)$, and the corresponding joint policy $\pi=($$Q$,$q^0$,$\gamma$,$\lambda)$.
  \begin{enumerate}
    \item A local history $h_i^t=(b_0, a_i^0, z_i^1, \ldots, a_i^{t-1},z_i^{t})$ is consistent with $\pi$ if the sequence of nodes $(q_i^0, q_i^1, \ldots, q_i^t)$ where $q_i^k = \lambda_i(q_i^{k-1},z_i^k)$ for $k=1, \ldots, t$ satisfies: $a_i^k = \gamma_i(q_i^k)$ for every $k$. We say $h_i^t$ ends at $q_i^t \in Q_i^t$ under $\pi$.
    \item A joint history $h^t = (h_1^t, \ldots, h_n^t)$ is consistent with $\pi$ if for all $i\in I$, $h_i^t$ is consistent with $\pi$ and ends at $q_i^{t}$.
    We say $h^t$ ends at $q^t=(q_1^t, \ldots, q_n^t) \in Q^t$ under $\pi$.
  \end{enumerate}
\end{definition}
Due to temporal consistency, any $h_i^{t}\in H_i^{t}$ consistent with a policy will end at some $q_i^t \in Q_i^{t}$.
Similarly, any $h^t\in H^t$ ends at some $q^t \in Q^{t}$.

\subsection{Node reachability probabilities} 
\label{sub:node_reachability_probabilities}
Above, we have defined when a history ends at a particular node.
Using this definition, we now derive the joint probability mass function (pmf) $P(q^t, h^t \mid \pi)$ of policy nodes and joint histories given that a particular policy $\pi$ is executed.

We note that $P(q^t, h^t \mid \pi) = P(q^t\mid h^t, \pi)P(h^t\mid \pi)$ and first consider $P(h^t \mid \pi)$.
The unconditional a priori probability of experiencing the joint history $h^0=(b^0)$ is $P(h^0)=1$.
For $t\geq 1$, the unconditional probability of experiencing $h^t$ is obtained recursively by $P(h^t) = \eta(z^t \mid \tau(h^{t-1}), a^{t-1})P(h^{t-1})$.
Conditioning $P(h^t)$ on a policy yields $P(h^t\mid \pi) = P(h^t)$ if $h^t$ is consistent with $\pi$ and 0 otherwise.
Next, we have $P(q^t \mid h^t, \pi) = \prod_{i\in I} P(q_i^t \mid h_i^t, \pi)$, with $P(q_i^t \mid h_i^t, \pi) = 1$ if $h_i^t$ ends at $q_i^t$ under $\pi$ and 0 otherwise.

Combining the above, the joint pmf is defined as
\[
P(q^t,h^t\mid \pi) = \begin{cases}
  P(h^t) &\text{if } h^t \text{ ends at } q^t  \text{ under } \pi\\
  0 & \text{otherwise}
\end{cases}.
\]
Marginalizing over $h^t$, the probability of ending at node $q^t$ under $\pi$ is
\begin{equation}
\label{eq:prob_of_node}
  P(q^t \mid \pi) = \sum\limits_{h^t \in H^t} P(q^t, h^t \mid \pi),
\end{equation}
and by definition of conditional probability,
\begin{equation}
\label{eq:prob_of_history_conditional_on_node}
  P(h^t \mid q^t, \pi) = \frac{P(q^t, h^t\mid \pi)}{P(q^t\mid \pi)}.
\end{equation}

We now find the probability of ending at $q_i^t$ under $\pi$.
Let $Q_{-i}^t$ denote the Cartesian product of all $Q_j^t$ \emph{except} $Q_i^t$.
Then $q_{-i}^t \in Q_{-i}^t$ denotes the nodes for all agents except $i$.
We have $(q_{-i}^t, q_i^t) \in Q^t$.
The probability of ending at $q_i^t$ under $\pi$ is
\begin{equation}
  P(q_i^t\mid \pi) = \sum\limits_{q_{-i}^t \in Q_{-i}^t} P\left((q_{-i}^t, q_{i}^t)\mid \pi\right),
\end{equation}
where the sum terms are determined by Eq.~\eqref{eq:prob_of_node}.
Again, by definition of conditional probability,
\begin{equation}
\label{eq:prob_of_other_nodes_conditional_on_individual_node}
  P(q_{-i}^t\mid q_i^t, \pi) = \frac{P\left((q_{-i}^t, q_i^t) \mid \pi\right)}{P(q_i^t\mid \pi)},
\end{equation}
where the term in the numerator is obtained from Eq.~\eqref{eq:prob_of_node}.

\subsection{Value of policy nodes} 
\label{sub:value_of_policy_nodes}
We define the values of a node in a joint policy and an individual policy.

\begin{definition}[Value of a joint policy node]
\label{def:value_of_node}
Given a joint policy $\pi=(Q,q^0,\gamma,\lambda)$, the value of a node $q^t \in Q^{t}$ is defined as
\begin{equation*}
V_t^\pi(q^t) = \mathbb{E}_{h_t \sim P(h^t \mid q^t, \pi)}\left[  V_t^\pi(\tau(h^t), q^t) \right],
\end{equation*}
where $P(h^t\mid q^t, \pi)$ is defined in Eq.~\eqref{eq:prob_of_history_conditional_on_node} and $\tau(h^t)$ is the joint belief corresponding to history $h^t$.
\end{definition}

\begin{definition}[Value of a local policy node]
\label{def:value_of_local_node}
For $i\in I$, let $\pi_i=(Q_i,q_i^0,\gamma_i,\lambda_i)$ be the local policy and let $\pi=(Q,q^0,\gamma,\lambda)$ be the corresponding joint policy.
For any $i \in I$, the value of a local node $q_i^t \in Q_i^{t}$ is
\begin{equation*}
    V_t^\pi(q_i^t) = \mathbb{E}_{q_{-i}^t \sim P(q_{-i}^t \mid q_i^t, \pi)}\left[ V_t^\pi\left((q_{-i}^t, q_i^t)\right)\right],
\end{equation*}
where $P(q_{-i}^t \mid q_i^t, \pi)$ is defined in Eq.~\eqref{eq:prob_of_other_nodes_conditional_on_individual_node}.
\end{definition}
In other words, the value of a local node $q_i^t$ is equal to the expected value of the value of the joint node $(q_{-i}^t, q_i^t)$ under $q_{-i}^t \sim P(q_{-i}^t \mid q_i^t, \pi)$.

\section{Convex-reward Dec-POMDPs} 
\label{sec:properties_of_information_gathering_Dec-POMDPs}
In this section, we prove several results for the value function of a Dec-POMDP whose reward function is convex in $\Delta(S)$.
Convex rewards are of special interest in information gathering.
This is because of their connection to so-called uncertainty functions~\cite{DeGroot2004}, which are non-negative functions concave in $\Delta(S)$.
Informally, an uncertainty function assigns large values to uncertain beliefs, and smaller values to less uncertain beliefs.
Negative uncertainty functions are convex and assign high values to less uncertain beliefs, and are thus suitable as reward functions for information gathering.
Examples of uncertainty functions include Shannon entropy, generalizations such as R{\'e}nyi entropy, and types of value of information, e.g., the probability of error in hypothesis testing.

The following theorem shows that if the immediate reward functions are convex in the joint belief, then the finite horizon value function of any policy is convex in the joint belief.
\begin{theorem}
\label{thm:convex}
  If the reward functions $\rho_T:\Delta(S)\to\mathbb{R}$ and $\rho_t:\Delta(S)\times A \to \mathbb{R}$ are convex in $\Delta(S)$, then for any policy $\pi$, $V_T:\Delta(S)\to\mathbb{R}$ is convex and $V_t^\pi : \Delta(S) \times Q^t \to \mathbb{R}$ is convex in $\Delta(S)$ for any $t$.
\end{theorem}
\begin{proof}
  Let $\pi=(Q,q^0,\gamma,\lambda)$, and $b\in \Delta(S)$.
  We proceed by induction ($V_T(b) = \rho_T(b)$ is trivial). 
  For $t=T-1$, let $q^{T-1} \in Q^{T-1}$, and denote $a := \gamma(q^{T-1})$.
  From Eq.~\eqref{eq:first_value_fun},
  $V_{T-1}^\pi(b,q^{T-1}) = \rho_{T-1}(b,a) + \sum\limits_{z\in Z}\eta\left(z\mid b, a\right)V_T\left(\zeta\left(b,a, z\right)\right)$.
  We recall from above that $V_T$ is convex, and by Eq.~\eqref{eq:bayes_filter}, the Bayes filter $\zeta\left(b,a, z\right)$ is a linear function of $b$.
  The composition of a linear and convex function is convex, so $V_T\left(\zeta\left(b,a, z\right)\right)$ is a convex function of $b$.
  The non-negative weighted sum of convex functions is also convex, and by assumption $\rho_{T-1}$ is convex in $\Delta(S)$, from which it follows that $V_{T-1}^\pi$ is convex in $\Delta(S)$.

  Now assume $V_{t+1}^\pi$ is convex in $\Delta(S)$ for some $1 \leq t \leq T-1$.
  By the definition in Eq.~\eqref{eq:value} and the same argumentation as above, it follows that $V_t^\pi$ is convex in $\Delta(S)$.
\end{proof}
Since a sufficiently large policy graph can represent any policy, we infer that the value function of an optimal policy is convex.

The following corollary gives a lower bound for the value of a policy graph node.
\begin{corollary}
\label{cor:lower_bound}
  Let $g^t: H^t \to [0, 1]$ be a probability mass function over the joint histories at time $t$.
  If the reward functions $\rho_T:\Delta(S)\to\mathbb{R}$ and $\rho_t:\Delta(S)\times A \to \mathbb{R}$ are convex in $\Delta(S)$, then for any time step $t$ and any policy $\pi$,
  \begin{equation*}
  \begin{split}
  \mathbb{E}_{h^t \sim g(h^t)} \left[V_t^\pi(\tau(h^t), q^t) \right] \geq V_t^\pi\left(\mathbb{E}_{h^t \sim g(h^t)}\left[\tau(h^t)\right], q^t\right).
  \end{split}  
  \end{equation*}
\end{corollary}
\begin{proof}
By Theorem~\ref{thm:convex}, $V_t^\pi: \Delta(S) \times Q^t \to \mathbb{R}$ is convex in $\Delta(S)$.
The claim immediately follows applying Jensen's inequality.
\end{proof}
Applied to Definition~\ref{def:value_of_node}, the corollary says the value of a joint policy node $q^t$ is lower bounded by the value of the expected joint belief at $q^t$.
Applied to Definition~\ref{def:value_of_local_node}, we obtain a lower bound for the value of a local policy node $q_i^t$ as
\begin{equation*}
    V_t^\pi(q_i^t) \geq \mathbb{E}_{q_{-i}^t \sim P(q_{-i}^t \mid q_{i}^t, \pi)} \left[ V_t^\pi\left( \mathbb{E}_{h^t \sim P(h^t \mid q^t, \pi)} \left[\tau(h^t) \right] , q^t\right) \right],
\end{equation*}
where inside the inner expectation we write $(q_{-i}^t,q_i^t)=q^t$.
Thus, we can evaluate a lower bound for the value of any local node $q_i^t \in Q_i^t$ by finding the values $V_t^\pi(q^t)$ of all joint nodes $q^t \in Q^t$ and then taking the expectation of $V_t^\pi(q^t)$ where $q^t=(q_{-i}^t,q_i^t)$ under $P(q_{-i}^t \mid q_i^t, \pi)$.

Corollary~\ref{cor:lower_bound} has applications in policy improvement algorithms that iteratively improve the value of a policy by modifying the output and node transition functions at each local policy node.
Instead of directly optimizing the value of a node, the lower bound can be optimized.
We present one such algorithm in the next section.

As Corollary~\ref{cor:lower_bound} holds for any pmf over joint histories, it could be applied also with pmfs other than $P(h^t \mid q^t, \pi)$.
For example, if it is expensive to enumerate the possible histories and beliefs at a node, one could approximate the lower bound, e.g., through importance sampling~\cite[Ch.~23.4]{Murphy2012}.

In standard Dec-POMDPs, the expected reward is a linear function of the joint belief.
Then, the corollary above holds with equality.
\begin{corollary}
\label{cor:state_dependent}
  Consider a Dec-POMDP where the reward functions are defined as $\rho_T(b) = \sum\limits_{s\in S}b(s)R_T(s)$ and for $0 \leq t \leq T-1$, $\rho_t(b,a) = \sum\limits_{s\in S} b(s)R_t(s,a)$, where $R_T:S\to \mathbb{R}$ is a state-dependent final reward function and $R_t:S\times A \to \mathbb{R}$ are the state-dependent reward functions.
  Then, the conclusion of Corollary~\ref{cor:lower_bound} holds with equality.
\end{corollary}
\begin{proof}
  Let $\pi=(Q,q_0,\gamma,\lambda)$ and $b\in \Delta(S)$.
  First note that $V_T(b) = \rho_T(b) = \sum\limits_{s\in S}b(s)R_T(s)$.
  Consider then $t=T-1$, and let $q^{T-1}\in Q^{T-1}$, and write $a:=\gamma(q^{T-1})$.
  Then from the definition of $V_{T-1}^\pi$ in Eq.~\eqref{eq:first_value_fun}, consider first the latter sum term which equals
  \begin{equation*}
    \begin{split}
      &\sum\limits_{z\in Z}\eta(z \mid b, a)\sum\limits_{s'\in S}\zeta(b,a,z)(s')R_T(s')\\
      =& \sum\limits_{s'\in S}\left[\sum\limits_{z\in Z}\sum\limits_{s \in S}P^z(z\mid s',a)P^s(s'\mid a,s)b(s)\right]R_T(s')
    \end{split}
  \end{equation*}
  which follows by replacing $\zeta(b,a,z)$ by Eq.~\eqref{eq:bayes_filter}, canceling out $\eta(z\mid b, a)$, and rearranging the sums. The above is clearly a linear function of $b$, and by definition, so is $\rho_t$, the first part of $V_{T-1}^\pi$.
  Thus, $V_{T-1}^\pi:\Delta(S) \times Q^{T-1} \to \mathbb{R}$ is linear in $\Delta(S)$.
  By an induction argument, it is now straightforward to show that $V_t^\pi$ is linear in $\Delta(S)$ for all $0 \leq t \leq T-1$.
  Finally,
  \begin{equation*}
    \mathbb{E}_{h^t \sim g(h^t)} \left[V_t^\pi(\tau(h^t), q^t) \right] = V_t^\pi\left(\mathbb{E}_{h^t \sim g(h^t)}\left[\tau(h^t)\right], q^t\right)
  \end{equation*}
  for any pmf $g$ over joint histories by linearity of expectation.
\end{proof}
Corollary~\ref{cor:state_dependent} shows that a solution algorithm for a Dec-POMDP with a reward convex in the joint belief that uses the lower bound from Corollary~\ref{cor:lower_bound} will also work for standard Dec-POMDPs with a reward linear in the joint belief.

Since a linear function is both convex and concave, rewards that are state-dependent and rewards that are convex in the joint belief can be combined on different time steps in one Dec-POMDP and the lower bound still holds.

\section{Policy graph improvement} 
\label{sec:policy_graph_improvement}
The Policy Graph Improvement (PGI) algorithm~\cite{Pajarinen2011} was originally introduced for standard Dec-POMDPs with reward function linear in the joint belief.
PGI monotonically improves policies by locally modifying the output and node transition functions of the individual agents' policies.
The policy size is fixed, such that the worst case computation time for an improvement iteration is known in advance. 
Moreover, due to the limited size of the policies the method produces compact, understandable policies.

We extend PGI to the non-linear reward case, and call the method non-linear PGI (NPGI).
Contrary to tree based Dec-POMDP approaches the policy does not grow double-exponentially with the planning horizon as we use a fixed size policy.
If the reward function is convex in $\Delta(S)$, NPGI may improve the lower bound from Corollary~\ref{cor:lower_bound}.
The lower bound is tight when each policy graph node corresponds to only one history suggesting we can improve the quality of the lower bound by increasing policy graph size.

NPGI is shown in Algorithm~\ref{alg:npgi}.
At each improvement step, NPGI repeats two steps: the forward pass and the backward pass.
In the forward pass, we use the current best joint policy to find the set $B$ of expected joint beliefs at every policy graph node.
In practice, we do this by first enumerating for each agent the sets of local histories ending at all local nodes, then taking the appropriate combinations to create the joint histories for joint policy graph nodes.
We then evaluate the expected joint beliefs at every joint policy graph node.

In the backward pass, we improve the current policy by modifying its output and node transition functions locally at each node.
As output from the backward pass, we obtain an updated policy $\pi^+$ using the improved output and node transition functions $\gamma^+$ and $\lambda^+$, respectively.
As NPGI may optimize a lower bound of the node values, we finally check if the value of the improved policy, $V_0^{\pi^+}(b^0,q^0)$, is greater than the value of the current best policy, and update the best policy if necessary.

\begin{algorithm}[t]
\caption{NPGI}
\label{alg:npgi}
\begin{algorithmic}[1]
\Require Policy $\pi=(Q,q^0,\gamma,\lambda)$, initial belief $b^0$
\Ensure Improved policy $\pi$
\While {not converged and time limit not exceeded}
  \State $B \gets $\Call{ForwardPass}{$\pi$, $b^0$}
  \State $\pi^+ \gets $\Call{BackwardPass}{$\pi$, $B$}
  \If {$V_0^{\pi^+}(b^0, q^0) \geq V_0^{\pi}(b^0, q^0)$}
     $\pi \gets \pi^+$ 
  \EndIf
\EndWhile
\State \Return $\pi$
\end{algorithmic}
\end{algorithm}

\begin{algorithm}[t]
\caption{BackwardPass}
\label{alg:backward}
\begin{algorithmic}[1]
\Require Policy $\pi=(Q,q^0,\gamma,\lambda)$, expected beliefs $B = \{b^q \mid q \in Q\}$
\Ensure Policy $\pi^+$ with improved output and node transition functions $\gamma^+$, $\lambda^+$
\State $\gamma^+ \gets \gamma, \lambda^+ \gets \lambda$
\For{$t = T-1, \ldots, 0$}
  \For{$i \in I$}
    \State $W_i^t \gets \emptyset$
    \For{$q_i^t \in Q_i^t$}
      \If{$t = T-1$}
      \State Solve Eq.~\eqref{eq:last_step_optimization}, assign $\gamma_i^+(q_i^t)$
      \Else
      \State Solve Eq.~\eqref{eq:step_optimization}, assign $\gamma_i^+(q_i^t)$ and $\lambda_i^+(q_i^t, z_i) \forall z_i$
      \EndIf
      \If{$\exists w_i^t \in W_i^t: $ \Call{SamePolicy}{$w_i^t$, $q_i^t$}} \label{line:check_same}
      \State \Call{Redirect}{$q_i^t$, $w_i^t$}
      \State \Call{Randomize}{$q_i^t$} \label{line:randomize}
      \EndIf
      \State $W_i^t \gets W_i^t \cup \{q_i^t\}$
    \EndFor
  \EndFor
\EndFor
\State \Return $(Q,q^0,\gamma^+, \lambda^+)$

\Procedure{SamePolicy}{$q_i^t,w_i^t$}
  \If{$\gamma_i^+(q_i^t) == \gamma_i^+(w_i^t) \wedge \forall z_i: \lambda_i^+(q_i^t,z_i)==\lambda_i^+(w_i^t,z_i)$}
    \State \Return True
  \Else 
    \State \Return False
  \EndIf
\EndProcedure
\Procedure{Redirect}{$q_i^t,w_i^t$}
  \For{$(x, z_i) \in \{(x,z_i) \in Q_i^{t-1}\times Z_i \mid \lambda_i^+(x,z_i) = q_i^t\}$}
    \State $\lambda_i^+(x,z_i) = w_i$
  \EndFor
\EndProcedure
\Procedure{Randomize}{$q_i^t$}
\State $\gamma_i^+(q_i^t) \sim \text{Uniform}(A_i)$
\If{$t \neq T-1$}
  \State $\forall z_i \in Z_i: \lambda_i^+(q_i^t, z_i) \sim \text{Uniform}(Q_i^{t+1})$
\EndIf
\EndProcedure
\end{algorithmic}
\end{algorithm}

\paragraph{Backward pass.}
The backward pass of NPGI is shown in Algorithm~\ref{alg:backward}.
At time step $t$ for agent $i$, for each node $q_i^t \in Q_i^t$, we maximize either the value $V_t^{\pi{^+}}(q_i^t)$ or its lower bound with respect to the local policy parameters.
In the following, we present the details for maximizing the lower bound, the algorithm for the exact value can be derived analogously.

For $t=T-1$, we consider the last remaining action.
Fix a local node $q_i^{T-1} \in Q_i^{T-1}$.
Denote the expected belief at $q^{T-1} = (q_1^{T-1}, \ldots, q_n^{T-1}) \in Q^{T-1}$ as $b := \mathbb{E}_{h^{T-1} \sim P(h^{T-1} \mid q^{T-1}, \pi)} \left[\tau(h^{T-1}) \right]$.\\
We write $a=\left(\gamma_1^+(q_1^{T-1}), \ldots, a_i^{T-1}, \ldots, \gamma_n^+(q_n^{T-1})\right) \in A$ as the joint action where local actions of all other agents except $i$ are fixed to those specified by the current output function.
We solve
\begin{equation}
\label{eq:last_step_optimization}
  \!\! \!\! \max\limits_{a_i^{T-1}\in A_i}\!\! \mathbb{E}\left[ \rho_{T-1}\left(b, a\right)  + \mathbb{E}\left[V_T(\zeta(b, a, z)\right]\right]
\end{equation}
where the outer expectation is under $q_{-i}^{T-1} \sim P(q_{-i}^{T-1} \mid q_i^{T-1}, \pi)$, the distribution over the nodes of agents other than $i$, and the inner expectation is under $\eta(z \mid b, a)$.
Note that in general, $b$ is different for each $q_{-i}^{T-1}$, as $q^{T-1}=(q_{-i}^{T-1},q_i^{T-1})$ will be different.
We assign $\gamma_i^+(q_i^{T-1})$ equal to the local action that maximizes Eq.~\eqref{eq:last_step_optimization}.
Note that this modification of the policy does not invalidate any of the expected beliefs at the nodes in $Q$.

For $t\leq T-1$, we consider both the current action and the next nodes via the node transition function.
Fix a local node $q_i^{t} \in Q_i^{t}$, and define $a$ and $b$ similarly as above.
Additionally, for any joint observation $z=(z_1, \ldots, z_n) \in Z$, define
\begin{equation*}
  q^{t+1}(z) = \left(\lambda_1^+(q_1^t, z_1), \ldots, q_i^{z_i}, \ldots, \lambda_n^+(q_n^t, z_n)\right)
\end{equation*}
as the next node in $Q^{t+1}$ when transitions of all other agents except $i$ are fixed to those specified by the current node transition function.
We solve
\begin{equation}
\begin{split}
\label{eq:step_optimization}
  \max\limits_{ \substack{a_i^t \in A_i \\ \forall z_i \in Z_i: q_i^{z_i} \in Q_i^{t+1}}} \mathbb{E} &\left[ \rho_t\left(b, a\right) + \mathbb{E}\left[V_{t+1}^{\pi{^+}}\left(\zeta(b,a,z), q^{t+1}(z)\right)\right] \right],
\end{split}
\end{equation}
where the outer expectation is under $q_{-i}^{t} \sim P(q_{-i}^{t} \mid q_i^{t}, \pi)$, and the inner expectation is under $\eta(z \mid b, a)$.
We assign $\gamma_i^+(q_i^t)$ and $\lambda_i^+(q_i^t, \cdot)$ to the respective maximizing values of Eq.~\eqref{eq:step_optimization}.
This assignment potentially invalidates the expected beliefs in $B$ for any nodes in $Q^k$ for $k \geq t+1$.
However, as in the subsequent optimization steps we only require the expected beliefs for $Q^k$, $k \leq t$, we do not need to repeat the forward pass.

Line~\ref{line:check_same} of Algorithm~\ref{alg:backward} checks if there exists a node $w_i^t$ that we have already optimized that has the same local policy as the current node $q_i^t$.
If such a node exists, we redirect all of the in-edges of $q_i^t$ to $w_i^t$ instead.
This redirection is required to maintain correct estimates of the respective node probabilities in the algorithm.
If we redirected the in-edges of $q_i^t$ to $w_i^t$, on Line~\ref{line:randomize} we randomize the local policy of the now useless node $q_i^t$ that has no in-edges\footnote{To randomize the local policy of a node $q_i^t\in Q_i^t$, we sample new local policies until we find one that is not identical to the local policy of any other node in $Q_i^t$. Likewise, when randomly initializing a new policy in our experiments we avoid including in any $Q_i^t$ nodes with identical local policies.}, in the hopes that it may be improved on subsequent backward passes.
If a node $q_i^t$ is to be improved that is unreachable, i.e., it has no in-edges or the probabilities of all histories ending in it are zero, we likewise randomize the local policy at that node.

\paragraph{Policy initialization.}
We initialize a random policy for each agent $i\in I$ with a given policy graph width $\left|Q_i^t\right|$ for each $t$ as follows\footnote{At the last time step, it is only meaningful to have $\left|Q_i^T\right| \leq \left|A_i\right|$. In our experiments if $\left|Q_i^T\right|>\left|A_i\right|$, we instead set $\left|Q_i^T\right|=\left|A_i\right|$.}.
For example, for a problem with $T=3$ and $\left|Q_i^t\right|=2$, we create a policy similar to Fig.~\ref{fig:pg_example} for each agent, where there is one initial node $q_i^0$, and 2 nodes at each time step $t\geq 1$.
The action determined by the output function $\gamma_i(q_i)$ is sampled uniformly at random from $A_i$.
For each node $q_i^t \in Q_i^t$ for $0 \leq t \leq T-1$, we sample a next node from $Q_i^{t+1}$ uniformly at random for each observation $z_i \in Z_i$ and assign the node transition function $\lambda_i(q_i,z_i)$ accordingly.

\section{Experiments} 
\label{sec:experiments}
We evaluate the performance of NPGI on information gathering Dec-POMDPs.
In the following, we introduce the problem domains, the experimental setup, and present the results.

\subsection{Problem domains} 
\label{sub:problem_domains}
We run experiments on the micro air vehicle (MAV) domain of~\cite{Lauri2017} and propose an information gathering rovers domain inspired by the Mars rovers domain of~\cite{Amato2009}.
In both tasks the objective of the agents is to maximize the expected sum of rewards collected minus the entropy of the joint belief at the end of the problem horizon.

\begin{figure}[t]
\begin{subfigure}[b]{0.45\columnwidth}
\centering
\begin{tikzpicture} [nodes in empty cells,
      nodes={minimum width=0.5cm, minimum height=0.5cm},
      row sep=-\pgflinewidth, column sep=-\pgflinewidth]
      border/.style={draw}
      \matrix(vector)[matrix of nodes,
          nodes={draw}]
      {
          $l_{0}$ & $l_{1}$ & $l_{2}$ & $l_{3}$\\
      };
      \node[left of=vector, xshift=-0.4cm] {MAV1};
      \node[right of=vector, xshift=0.4cm] {MAV2};
\end{tikzpicture}
\end{subfigure}~
\begin{subfigure}[b]{0.45\columnwidth}
\centering
\begin{tikzpicture} [nodes in empty cells,
      nodes={minimum width=0.5cm, minimum height=0.5cm},
      row sep=-\pgflinewidth, column sep=-\pgflinewidth]
      border/.style={draw}
      \matrix(vector)[matrix of nodes,
          nodes={draw}]
      {
          $l_{0}$ & $l_{2}$ \\
          $l_{1}$ & $l_{3}$ \\
      };
\end{tikzpicture}
\end{subfigure}
\caption{Arrangement of locations in the MAV domain (left) and rovers domain (right).} 
\label{fig:domains}
\end{figure}
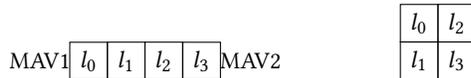

\paragraph{MAV domain.}
A target moves between four possible locations, $l_i$ in Figure~\ref{fig:domains}.
The target is either friendly or hostile; a hostile target moves more aggressively.
Two MAVs, MAV1 and MAV2 in the figure, are tasked with tracking the target and inferring whether it is friendly or hostile.
The MAVs can choose to use either a camera or a radar sensor to sense the location of the target.
An observation from either sensor corresponds to a noisy measurement of the target's location.
The camera is more accurate if the target is close, whereas the radar is more accurate when the target is further away.
The Manhattan distance is applied, i.e., at $l_0$ the target is at distance 0 from MAV1 and at distance 3 from MAV2.
If both MAVs apply their radars simultaneously, accuracy decreases due to interference.

Using the camera has zero cost, and using the radar sensor has a cost of 0.1, and an additional cost of 1 or 0.1 if the target is at distance 0 or 1 to the MAV, respectively, to model the risk of revealing the MAVs own location to the (potentially hostile) target.
To model information gathering, we set the final reward equal to the negative Shannon entropy of the joint belief, i.e., $\rho_T(b) = \sum\limits_{s\in S} b(s)\log_2 b(s)$.
The initial belief is a uniform over all states.
This problem has 8 states; 4 target locations and a binary variable for friendly/hostile, and 2 actions and 4 observations per agent.

\paragraph{Information gathering rovers.}
Two rovers are collecting information on four sites $l_i$ of interest arranged as shown in Figure~\ref{fig:domains}.
Each site is in one of two possible states which remains fixed throughout.
The agents can move north, south, east, or west.
Movement fails with probability 0.2, in which case the agent remains at its current location.
The agents always fully observe their own location.
Additionally, the agents can choose to conduct measurements of the site they are currently at.
A binary measurement of the site status is recorded with false positive and false negative probabilities of 0.2 each.
If the agents measure at the same location at the same time, the false positive and false negative probabilities are significantly lower, 0.05 and 0.01, respectively.
Movement has zero cost, while measuring has a cost of 0.1.
The final reward is equal to the negative entropy.
The initial belief is such that one agent starts at $l_0$, the other at $l_3$, with a uniform belief over the site status.
The problem has 256 states, and 5 local actions and 8 local observations per agent.

\subsection{Experimental setup} 
\label{sub:experimental_setup}
We compare NPGI to one exact algorithm and two heuristic algorithms.
The exact method we employ is the Generalized Multi-Agent A* with incremental expansion (GMAA*-ICE)~\cite{Oliehoek2013} with the QPOMDP search heuristic.
According to~\cite{Oliehoek2013} a vector representation of the search heuristic, analogous to the representation of an optimal POMDP value function by a set of so-called $\alpha$-vectors~\cite{Smallwood1973}, can help scale up GMAA*-ICE to larger problems.
However, since the vector representation only exists if the reward function is linear in the joint belief, we represent the search heuristic as a tree.
The two heuristic methods are joint equilibrium based search for policies (JESP)~\cite{Nair2003} and direct cross-entropy policy search (DICEPS)~\cite{Oliehoek2008}.

All of the methods above are easily modified to our domains where the final reward is equal to the negative Shannon entropy.
However, applicability of NPGI is wider as it allows the reward at \emph{any} time step to be a convex function of the joint belief.
We note that there are other algorithms such as FB-HSVI~\cite{Dibangoye2016} and PBVI-BB~\cite{MacDermed2013} that have demonstrated good performance on many benchmarks.
However, these algorithms rely on linearity of the reward to achieve compression of histories and joint beliefs, and non-trivial modifications beyond the scope of this work would be required to extend them to Dec-POMDPs with non-linear rewards.

As baselines, we report values of a greedy open loop policy that executes a sequence of joint actions that has the maximal expected sum of rewards under the initial belief, and the best blind policy that always executes the same joint action.

We run NPGI using both the exact value of nodes and the lower bound from Corollary~\ref{cor:lower_bound}.
The number of policy graph nodes $\left|Q_i^t\right|$ at each time step $t$ is 2, 3, or 4.
For each run with NPGI we run 30 backward passes, starting from randomly sampled initial policies.
For all methods, we report the averages over 100 runs.
If a run does not finish in 2 hours, we terminate it.

\subsection{Results} 
\label{sub:Dec-POMDP_tasks}
Tables~\ref{tab:mav} and~\ref{tab:rover} show the average policy values in the MAV and information gathering rovers problems, respectively.
NPGI is indicated by ``Ours'' when the lower bound (LB) was used, and as ``Ours (No LB)'' when exact evaluation of node values was applied.
Results are reported as function of the problem horizon $T$, and for NPGI also as function of the policy graph size $\left|Q_i^t\right|$.
The symbol ``-'' indicates missing results due to exceeding the cut-off time.

GMAA*-ICE finds an optimal solution, but similarly to~\cite{Lauri2017} we find that it does not scale beyond $T=3$ in either problem.
Considering $T=2$ and $T=3$, the average values of our method are very close to the optimal value in both problems.
In these cases, we found that NPGI finds an optimal policy in all the MAV problem runs, and in about 60\% of the MAV problem runs.

\begin{table}[t]
\caption{Average policy values in the MAV domain ($|S|=8$, $\left|A_i\right|=2$, $\left|Z_i\right|=4$).}
\label{tab:mav}
\begin{tabular}{@{}lcllll@{}}
\toprule
Method         & $\left|Q_i^t\right|$    & $T=2$  & $T=3$  & $T=4$  & $T=5$  \\ \midrule
               & 2            & -1.919 & -1.831 & -1.768 & -1.725 \\
Ours           & 3            & -1.919 & -1.831 & -1.768 & -1.725 \\
               & 4            & -1.919 & -1.831 & -1.768 & -1.725 \\ \midrule
               & 2            & -1.919 & -1.831 & -1.768 & -1.726 \\
Ours (No LB)   & 3            & -1.919 & -1.831 & -1.768 & -1.725 \\
               & 4            & -1.919 & -1.831 & -1.768 & -1.726 \\ \midrule
\multicolumn{2}{l}{DICEPS}    & -1.925 & -1.937 & -1.926 & -1.940 \\
\multicolumn{2}{l}{JESP}      & -1.953 & -1.859 & -1.794 & -1.750 \\
\multicolumn{2}{l}{GMAA*-ICE} & -1.919 & -1.831 & -      & -      \\
\multicolumn{2}{l}{Greedy}    & -2.156 & -2.044 & -1.978 & -1.932 \\
\multicolumn{2}{l}{Blind}     & -1.945 & -1.904 & -1.909 & -1.932 \\ \bottomrule
\end{tabular}
\end{table}

\begin{table}[t]
\caption{Average policy values in the information gathering rovers domain ($|S|=256$, $\left|A_i\right|=5$, $\left|Z_i\right|=8$). }
\label{tab:rover}
\begin{tabular}{@{}lcllll@{}}
\toprule
Method         & $\left|Q_i^t\right|$  & $T=2$  & $T=3$  & $T=4$  & $T=5$  \\ \midrule
               & 2            & -3.495 & -3.189 & -3.034 & -2.989 \\
Ours           & 3            & -3.498 & -3.189 & -3.034 & -2.977 \\
               & 4            & -3.500 & -3.189 & -3.034 & -3.004 \\ \midrule
               & 2            & -3.495 & -3.189 & -3.035 & -2.976 \\
Ours (No LB)   & 3            & -3.498 & -3.189 & -3.035 & -3.085 \\
               & 4            & -3.500 & -3.189 & -3.035 & -      \\ \midrule
\multicolumn{2}{l}{DICEPS}    & -3.482 & -3.535 & -3.825 & -4.792 \\
\multicolumn{2}{l}{JESP}      & -3.483 & -3.536 & -      & -      \\
\multicolumn{2}{l}{GMAA*-ICE} & -3.479 & -3.189 & -      & -      \\
\multicolumn{2}{l}{Greedy}    & -3.844 & -4.031 & -3.877 & -3.818 \\
\multicolumn{2}{l}{Blind}     & -3.479 & -3.412 & -3.418 & -3.472 \\ \bottomrule
\end{tabular}
\end{table}

\begin{table}[t]
\caption{Average NPGI backward pass duration (in seconds) with or without lower bound (LB).}
\label{tab:improvement_times}
\begin{tabular}{@{}crrrr@{}}
\toprule
  & \multicolumn{2}{c}{MAV}  &  \multicolumn{2}{c}{Rovers}  \\ \cmidrule(lr){2-3} \cmidrule(lr){4-5}
$T$ & With LB & No LB     & With LB & No LB   \\ \midrule
2   & 0.002     & 0.002   & 0.04    & 0.04    \\
3   & 0.04      & 0.05    & 0.26    & 0.34    \\
4   & 1.20      & 2.74    & 1.43    & 4.40    \\
5   & 31.02     & 55.34   & 32.23   & 158.7   \\ \bottomrule
\end{tabular}
\end{table}

In the MAV problem (Table~\ref{tab:mav}), performance of our method is consistent for varying policy graph size $\left|Q_i^t\right|$ and horizon $T$.
This indicates that even a small policy suffices to reach a high value in this problem.
We also note that applying the lower bound does not reduce the quality of the policy found by our approach.

In the rover problem (Table~\ref{tab:rover}), we observe more variation in policy quality as function of the policy graph size.
However, applying the Mann-Whitney U-test we do not find significant differences (significance level of 0.01) either for varying policy graph size, nor for exact computation versus applying the lower bound.
A compact policy with as few as 2 nodes per time step in the policy graph can reach a high value in this problem as well.

Table~\ref{tab:improvement_times} shows the average duration of one backward pass of Algorithm~\ref{alg:npgi} as function of the problem horizon $T$ with $\left|Q_i^t\right|=2$, with or without using the lower bound (LB).
The lower runtime requirement when applying the lower bound is seen clearly for $T\geq 4$.
The runtime of NPGI is dominated by the backward pass and solving the local policy optimization problems, Eqns.~\eqref{eq:last_step_optimization} and~\eqref{eq:step_optimization}, which applying the lower bound help reduce.
As indicated by the results in Tables~\ref{tab:mav} and~\ref{tab:rover}, applying the lower bound also does not degrade the quality of the policies found.

Our method outperforms the baselines except for $T=2$ in the rover problem where a blind policy of always measuring is optimal.
In several cases, JESP and DICEPS return policies with a value lower than one of both of the baselines.

The size of the policy graph in NPGI must be specified before calculating the policy.
As shown by our experiments, fixing the policy graph size effectively limits the space of policies to be explored and can produce compact and understandable policies.
However, a potential weakness is that optimizing over fixed-size policies excludes the possibility to find a larger but potentially better policy.

\section{Conclusion} 
\label{sec:conclusion}
We showed that if the reward function in a finite-horizon Dec-POMDP is convex in the joint belief, then the value function of any policy is then convex in the joint belief.
Rewards that are convex in the joint belief are of importance in information gathering problems.
We applied the result to derive a lower bound for the value, and empirically demonstrated that it improves the run-time of a heuristic planning algorithm without degrading solution quality.

We presented the first heuristic algorithm for Dec-POMDPs with rewards convex in the joint belief, and showed that it reaches good performance in large Dec-POMDPs.
Future work includes developing an approximation algorithm with bounded suboptimality.
Approximation of the reward function by a piecewise linear function similar to~\cite{Araya2010} is a potential first step towards this goal.

\begin{acks}
J.P. and J.P. were supported by \grantsponsor{erc}{European Research Council}{https://erc.europa.eu/} under Grant No.~\grantnum{erc}{640554} (SKILLS4ROBOTS) and \grantsponsor{dfg}{German Research Foundation}{http://www.dfg.de/en/} project~\grantnum{dfg}{PA 3179/1-1} (ROBOLEAP).
\end{acks}


\bibliographystyle{ACM-Reference-Format}  
\balance  
\bibliography{references}  


\begin{thebibliography}{00}


\ifx \showCODEN    \undefined \def \showCODEN     #1{\unskip}     \fi
\ifx \showDOI      \undefined \def \showDOI       #1{#1}\fi
\ifx \showISBNx    \undefined \def \showISBNx     #1{\unskip}     \fi
\ifx \showISBNxiii \undefined \def \showISBNxiii  #1{\unskip}     \fi
\ifx \showISSN     \undefined \def \showISSN      #1{\unskip}     \fi
\ifx \showLCCN     \undefined \def \showLCCN      #1{\unskip}     \fi
\ifx \shownote     \undefined \def \shownote      #1{#1}          \fi
\ifx \showarticletitle \undefined \def \showarticletitle #1{#1}   \fi
\ifx \showURL      \undefined \def \showURL       {\relax}        \fi
\providecommand\bibfield[2]{#2}
\providecommand\bibinfo[2]{#2}
\providecommand\natexlab[1]{#1}
\providecommand\showeprint[2][]{arXiv:#2}

\bibitem[\protect\citeauthoryear{Allen and Zilberstein}{Allen and
  Zilberstein}{2009}]%
        {allen2009complexity}
\bibfield{author}{\bibinfo{person}{Martin Allen} {and} \bibinfo{person}{Shlomo
  Zilberstein}.} \bibinfo{year}{2009}\natexlab{}.
\newblock \showarticletitle{Complexity of decentralized control: Special
  cases}. In \bibinfo{booktitle}{{\em Advances in Neural Information Processing
  Systems (NIPS)}}. \bibinfo{pages}{19--27}.
\newblock


\bibitem[\protect\citeauthoryear{Amato and Zilberstein}{Amato and
  Zilberstein}{2009}]%
        {Amato2009}
\bibfield{author}{\bibinfo{person}{Christopher Amato} {and}
  \bibinfo{person}{Shlomo Zilberstein}.} \bibinfo{year}{2009}\natexlab{}.
\newblock \showarticletitle{Achieving goals in decentralized POMDPs}. In
  \bibinfo{booktitle}{{\em Autonomous Agents and Multiagent Systems (AAMAS)}}.
  \bibinfo{pages}{593--600}.
\newblock


\bibitem[\protect\citeauthoryear{Araya-L{\'o}pez, Buffet, Thomas, and
  Charpillet}{Araya-L{\'o}pez et~al\mbox{.}}{2010}]%
        {Araya2010}
\bibfield{author}{\bibinfo{person}{Mauricio Araya-L{\'o}pez},
  \bibinfo{person}{Olivier Buffet}, \bibinfo{person}{Vincent Thomas}, {and}
  \bibinfo{person}{Francois Charpillet}.} \bibinfo{year}{2010}\natexlab{}.
\newblock \showarticletitle{{A POMDP Extension with Belief-dependent Rewards}}.
  In \bibinfo{booktitle}{{\em Advances in Neural Information Processing Systems
  (NIPS)}}. \bibinfo{pages}{64--72}.
\newblock


\bibitem[\protect\citeauthoryear{Bernstein, Givan, Immerman, and
  Zilberstein}{Bernstein et~al\mbox{.}}{2002}]%
        {Bernstein2002}
\bibfield{author}{\bibinfo{person}{Daniel~S Bernstein}, \bibinfo{person}{Robert
  Givan}, \bibinfo{person}{Neil Immerman}, {and} \bibinfo{person}{Shlomo
  Zilberstein}.} \bibinfo{year}{2002}\natexlab{}.
\newblock \showarticletitle{The complexity of decentralized control of Markov
  decision processes}.
\newblock \bibinfo{journal}{{\em Mathematics of operations research\/}}
  \bibinfo{volume}{27}, \bibinfo{number}{4} (\bibinfo{year}{2002}),
  \bibinfo{pages}{819--840}.
\newblock


\bibitem[\protect\citeauthoryear{Charrow, Kumar, and Michael}{Charrow
  et~al\mbox{.}}{2014}]%
        {Charrow2014}
\bibfield{author}{\bibinfo{person}{Benjamin Charrow}, \bibinfo{person}{Vijay
  Kumar}, {and} \bibinfo{person}{Nathan Michael}.}
  \bibinfo{year}{2014}\natexlab{}.
\newblock \showarticletitle{Approximate representations for multi-robot control
  policies that maximize mutual information}.
\newblock \bibinfo{journal}{{\em Autonomous Robots\/}} \bibinfo{volume}{37},
  \bibinfo{number}{4} (\bibinfo{year}{2014}), \bibinfo{pages}{383--400}.
\newblock


\bibitem[\protect\citeauthoryear{DeGroot}{DeGroot}{2004}]%
        {DeGroot2004}
\bibfield{author}{\bibinfo{person}{Morris~H DeGroot}.}
  \bibinfo{year}{2004}\natexlab{}.
\newblock \bibinfo{booktitle}{{\em Optimal Statistical Decisions}}.
\newblock \bibinfo{publisher}{John Wiley \& Sons, Inc.},
  \bibinfo{address}{Hoboken, NJ}.
\newblock
\newblock
\shownote{Wiley Classics Library edition.}


\bibitem[\protect\citeauthoryear{Dibangoye, Amato, Buffet, and
  Charpillet}{Dibangoye et~al\mbox{.}}{2016}]%
        {Dibangoye2016}
\bibfield{author}{\bibinfo{person}{Jilles~Steeve Dibangoye},
  \bibinfo{person}{Christopher Amato}, \bibinfo{person}{Olivier Buffet}, {and}
  \bibinfo{person}{Fran{\c{c}}ois Charpillet}.}
  \bibinfo{year}{2016}\natexlab{}.
\newblock \showarticletitle{Optimally solving Dec-POMDPs as continuous-state
  MDPs}.
\newblock \bibinfo{journal}{{\em Journal of Artificial Intelligence
  Research\/}}  \bibinfo{volume}{55} (\bibinfo{year}{2016}),
  \bibinfo{pages}{443--497}.
\newblock


\bibitem[\protect\citeauthoryear{Fehr, Buffet, Thomas, and Dibangoye}{Fehr
  et~al\mbox{.}}{2018}]%
        {Fehr2018}
\bibfield{author}{\bibinfo{person}{Mathieu Fehr}, \bibinfo{person}{Olivier
  Buffet}, \bibinfo{person}{Vincent Thomas}, {and} \bibinfo{person}{Jilles
  Dibangoye}.} \bibinfo{year}{2018}\natexlab{}.
\newblock \showarticletitle{rho-POMDPs have Lipschitz-Continuous
  epsilon-Optimal Value Functions}.
\newblock In \bibinfo{booktitle}{{\em Advances in Neural Information Processing
  Systems (NIPS)}}. \bibinfo{pages}{6933--6943}.
\newblock


\bibitem[\protect\citeauthoryear{Hansen, Bernstein, and Zilberstein}{Hansen
  et~al\mbox{.}}{2004}]%
        {Hansen2004}
\bibfield{author}{\bibinfo{person}{Eric~A Hansen}, \bibinfo{person}{Daniel~S
  Bernstein}, {and} \bibinfo{person}{Shlomo Zilberstein}.}
  \bibinfo{year}{2004}\natexlab{}.
\newblock \showarticletitle{Dynamic programming for partially observable
  stochastic games}. In \bibinfo{booktitle}{{\em AAAI}}.
  \bibinfo{pages}{709--715}.
\newblock


\bibitem[\protect\citeauthoryear{Lauri, Hein{\"a}nen, and Frintrop}{Lauri
  et~al\mbox{.}}{2017}]%
        {Lauri2017}
\bibfield{author}{\bibinfo{person}{Mikko Lauri}, \bibinfo{person}{Eero
  Hein{\"a}nen}, {and} \bibinfo{person}{Simone Frintrop}.}
  \bibinfo{year}{2017}\natexlab{}.
\newblock \showarticletitle{Multi-robot active information gathering with
  periodic communication}. In \bibinfo{booktitle}{{\em IEEE Intl. Conference on
  Robotics and Automation (ICRA)}}. \bibinfo{pages}{851--856}.
\newblock


\bibitem[\protect\citeauthoryear{MacDermed and Isbell}{MacDermed and
  Isbell}{2013}]%
        {MacDermed2013}
\bibfield{author}{\bibinfo{person}{Liam~C MacDermed} {and}
  \bibinfo{person}{Charles~L Isbell}.} \bibinfo{year}{2013}\natexlab{}.
\newblock \showarticletitle{Point Based Value Iteration with Optimal Belief
  Compression for Dec-POMDPs}.
\newblock In \bibinfo{booktitle}{{\em Advances in Neural Information Processing
  Systems (NIPS)}}. \bibinfo{pages}{100--108}.
\newblock


\bibitem[\protect\citeauthoryear{Murphy}{Murphy}{2012}]%
        {Murphy2012}
\bibfield{author}{\bibinfo{person}{Kevin Murphy}.}
  \bibinfo{year}{2012}\natexlab{}.
\newblock \bibinfo{booktitle}{{\em {Machine Learning: A probabilistic
  perspective}}}.
\newblock \bibinfo{publisher}{MIT Press}.
\newblock


\bibitem[\protect\citeauthoryear{Nair, Tambe, Yokoo, Pynadath, and
  Marsella}{Nair et~al\mbox{.}}{2003}]%
        {Nair2003}
\bibfield{author}{\bibinfo{person}{Ranjit Nair}, \bibinfo{person}{Milind
  Tambe}, \bibinfo{person}{Makoto Yokoo}, \bibinfo{person}{David Pynadath},
  {and} \bibinfo{person}{Stacy Marsella}.} \bibinfo{year}{2003}\natexlab{}.
\newblock \showarticletitle{Taming decentralized POMDPs: Towards efficient
  policy computation for multiagent settings}. In \bibinfo{booktitle}{{\em
  Intl. Joint Conference on Artificial Intelligence (IJCAI)}}.
  \bibinfo{pages}{705--711}.
\newblock


\bibitem[\protect\citeauthoryear{Oliehoek}{Oliehoek}{2013}]%
        {Oliehoek2013b}
\bibfield{author}{\bibinfo{person}{Frans~A Oliehoek}.}
  \bibinfo{year}{2013}\natexlab{}.
\newblock \showarticletitle{{Sufficient Plan-Time Statistics for Decentralized
  POMDPs}}. In \bibinfo{booktitle}{{\em Intl. Joint Conference on Artificial
  Intelligence (IJCAI)}}. \bibinfo{pages}{302--308}.
\newblock


\bibitem[\protect\citeauthoryear{Oliehoek and Amato}{Oliehoek and
  Amato}{2016}]%
        {Oliehoek2016}
\bibfield{author}{\bibinfo{person}{Frans~A Oliehoek} {and}
  \bibinfo{person}{Christopher Amato}.} \bibinfo{year}{2016}\natexlab{}.
\newblock \bibinfo{booktitle}{{\em {A Concise Introduction to Decentralized
  POMDPs}}}.
\newblock \bibinfo{publisher}{Springer}.
\newblock


\bibitem[\protect\citeauthoryear{Oliehoek, Kooij, and Vlassis}{Oliehoek
  et~al\mbox{.}}{2008a}]%
        {Oliehoek2008}
\bibfield{author}{\bibinfo{person}{Frans~A Oliehoek},
  \bibinfo{person}{Julian~FP Kooij}, {and} \bibinfo{person}{Nikos Vlassis}.}
  \bibinfo{year}{2008}\natexlab{a}.
\newblock \showarticletitle{The cross-entropy method for policy search in
  decentralized POMDPs}.
\newblock \bibinfo{journal}{{\em Informatica\/}} \bibinfo{volume}{32},
  \bibinfo{number}{4} (\bibinfo{year}{2008}), \bibinfo{pages}{341--357}.
\newblock


\bibitem[\protect\citeauthoryear{Oliehoek, Spaan, Amato, and Whiteson}{Oliehoek
  et~al\mbox{.}}{2013}]%
        {Oliehoek2013}
\bibfield{author}{\bibinfo{person}{Frans~A Oliehoek},
  \bibinfo{person}{Matthijs~TJ Spaan}, \bibinfo{person}{Christopher Amato},
  {and} \bibinfo{person}{Shimon Whiteson}.} \bibinfo{year}{2013}\natexlab{}.
\newblock \showarticletitle{Incremental clustering and expansion for faster
  optimal planning in Dec-POMDPs}.
\newblock \bibinfo{journal}{{\em Journal of Artificial Intelligence
  Research\/}}  \bibinfo{volume}{46} (\bibinfo{year}{2013}),
  \bibinfo{pages}{449--509}.
\newblock


\bibitem[\protect\citeauthoryear{Oliehoek, Spaan, Whiteson, and
  Vlassis}{Oliehoek et~al\mbox{.}}{2008b}]%
        {Oliehoek2008b}
\bibfield{author}{\bibinfo{person}{Frans~A Oliehoek},
  \bibinfo{person}{Matthijs~TJ Spaan}, \bibinfo{person}{Shimon Whiteson}, {and}
  \bibinfo{person}{Nikos Vlassis}.} \bibinfo{year}{2008}\natexlab{b}.
\newblock \showarticletitle{Exploiting locality of interaction in factored
  Dec-POMDPs}. In \bibinfo{booktitle}{{\em Autonomous Agents and Multiagent
  Systems (AAMAS)}}. \bibinfo{pages}{517--524}.
\newblock


\bibitem[\protect\citeauthoryear{Pajarinen and Peltonen}{Pajarinen and
  Peltonen}{2011}]%
        {Pajarinen2011}
\bibfield{author}{\bibinfo{person}{Joni~K Pajarinen} {and}
  \bibinfo{person}{Jaakko Peltonen}.} \bibinfo{year}{2011}\natexlab{}.
\newblock \showarticletitle{Periodic Finite State Controllers for Efficient
  POMDP and DEC-POMDP Planning}.
\newblock In \bibinfo{booktitle}{{\em Advances in Neural Information Processing
  Systems (NIPS)}}. \bibinfo{pages}{2636--2644}.
\newblock


\bibitem[\protect\citeauthoryear{Satsangi, Whiteson, Oliehoek, and
  Spaan}{Satsangi et~al\mbox{.}}{2018}]%
        {Satsangi2018}
\bibfield{author}{\bibinfo{person}{Yash Satsangi}, \bibinfo{person}{Shimon
  Whiteson}, \bibinfo{person}{Frans~A Oliehoek}, {and}
  \bibinfo{person}{Matthijs~TJ Spaan}.} \bibinfo{year}{2018}\natexlab{}.
\newblock \showarticletitle{Exploiting submodular value functions for scaling
  up active perception}.
\newblock \bibinfo{journal}{{\em Autonomous Robots\/}} \bibinfo{volume}{42},
  \bibinfo{number}{2} (\bibinfo{year}{2018}), \bibinfo{pages}{209--233}.
\newblock


\bibitem[\protect\citeauthoryear{Schlotfeldt, Thakur, Atanasov, Kumar, and
  Pappas}{Schlotfeldt et~al\mbox{.}}{2018}]%
        {Schlotfeldt2018}
\bibfield{author}{\bibinfo{person}{Brent Schlotfeldt}, \bibinfo{person}{Dinesh
  Thakur}, \bibinfo{person}{Nikolay Atanasov}, \bibinfo{person}{Vijay Kumar},
  {and} \bibinfo{person}{George~J Pappas}.} \bibinfo{year}{2018}\natexlab{}.
\newblock \showarticletitle{Anytime Planning for Decentralized Multirobot
  Active Information Gathering}.
\newblock \bibinfo{journal}{{\em IEEE Robotics and Automation Letters\/}}
  \bibinfo{volume}{3}, \bibinfo{number}{2} (\bibinfo{year}{2018}),
  \bibinfo{pages}{1025--1032}.
\newblock


\bibitem[\protect\citeauthoryear{Seuken and Zilberstein}{Seuken and
  Zilberstein}{2007}]%
        {Seuken2007}
\bibfield{author}{\bibinfo{person}{Sven Seuken} {and} \bibinfo{person}{Shlomo
  Zilberstein}.} \bibinfo{year}{2007}\natexlab{}.
\newblock \showarticletitle{Memory-Bounded Dynamic Programming for
  DEC-POMDPs.}. In \bibinfo{booktitle}{{\em Intl. Joint Conference on
  Artificial Intelligence (IJCAI)}}. \bibinfo{pages}{2009--2015}.
\newblock


\bibitem[\protect\citeauthoryear{Smallwood and Sondik}{Smallwood and
  Sondik}{1973}]%
        {Smallwood1973}
\bibfield{author}{\bibinfo{person}{Richard~D Smallwood} {and}
  \bibinfo{person}{Edward~J Sondik}.} \bibinfo{year}{1973}\natexlab{}.
\newblock \showarticletitle{The optimal control of partially observable Markov
  processes over a finite horizon}.
\newblock \bibinfo{journal}{{\em Operations research\/}} \bibinfo{volume}{21},
  \bibinfo{number}{5} (\bibinfo{year}{1973}), \bibinfo{pages}{1071--1088}.
\newblock


\bibitem[\protect\citeauthoryear{Spaan, Gordon, and Vlassis}{Spaan
  et~al\mbox{.}}{2006}]%
        {Spaan2006}
\bibfield{author}{\bibinfo{person}{Matthijs~TJ Spaan},
  \bibinfo{person}{Geoffrey~J Gordon}, {and} \bibinfo{person}{Nikos Vlassis}.}
  \bibinfo{year}{2006}\natexlab{}.
\newblock \showarticletitle{Decentralized planning under uncertainty for teams
  of communicating agents}. In \bibinfo{booktitle}{{\em Autonomous Agents and
  Multiagent Systems (AAMAS)}}. \bibinfo{pages}{249--256}.
\newblock


\bibitem[\protect\citeauthoryear{Spaan, Veiga, and Lima}{Spaan
  et~al\mbox{.}}{2015}]%
        {Spaan2015}
\bibfield{author}{\bibinfo{person}{Matthijs~TJ Spaan}, \bibinfo{person}{Tiago~S
  Veiga}, {and} \bibinfo{person}{Pedro~U Lima}.}
  \bibinfo{year}{2015}\natexlab{}.
\newblock \showarticletitle{Decision-theoretic planning under uncertainty with
  information rewards for active cooperative perception}.
\newblock \bibinfo{journal}{{\em Autonomous Agents and Multi-Agent Systems\/}}
  \bibinfo{volume}{29}, \bibinfo{number}{6} (\bibinfo{year}{2015}),
  \bibinfo{pages}{1157--1185}.
\newblock


\bibitem[\protect\citeauthoryear{Szer, Charpillet, and Zilberstein}{Szer
  et~al\mbox{.}}{2005}]%
        {Szer2005}
\bibfield{author}{\bibinfo{person}{Daniel Szer},
  \bibinfo{person}{Fran{\c{c}}ois Charpillet}, {and} \bibinfo{person}{Shlomo
  Zilberstein}.} \bibinfo{year}{2005}\natexlab{}.
\newblock \showarticletitle{MAA*: a heuristic search algorithm for solving
  decentralized POMDPs}. In \bibinfo{booktitle}{{\em Uncertainty in Artificial
  Intelligence (UAI)}}. \bibinfo{pages}{576--583}.
\newblock


\bibitem[\protect\citeauthoryear{Wu, Zilberstein, and Chen}{Wu
  et~al\mbox{.}}{2011}]%
        {Wu2011}
\bibfield{author}{\bibinfo{person}{Feng Wu}, \bibinfo{person}{Shlomo
  Zilberstein}, {and} \bibinfo{person}{Xiaoping Chen}.}
  \bibinfo{year}{2011}\natexlab{}.
\newblock \showarticletitle{Online planning for multi-agent systems with
  bounded communication}.
\newblock \bibinfo{journal}{{\em Artificial Intelligence\/}}
  \bibinfo{volume}{175}, \bibinfo{number}{2} (\bibinfo{year}{2011}),
  \bibinfo{pages}{487--511}.
\newblock


\end{thebibliography}

\end{document}